\documentclass[11pt]{article}

\usepackage[preprint]{acl}

\usepackage{times}
\usepackage{latexsym}
\usepackage[T1]{fontenc}
\usepackage[utf8]{inputenc}
\usepackage{microtype}
\usepackage{inconsolata}
\usepackage{graphicx}
\usepackage{amsmath}
\usepackage{amssymb}
\usepackage{amsthm}
\usepackage{booktabs}
\usepackage{multirow}
\usepackage{makecell}
\usepackage{xcolor}
\usepackage{algorithm}
\usepackage{algorithmic}
\usepackage{pifont}
\usepackage{tikz}
\usepackage{stfloats}
\usetikzlibrary{decorations.pathreplacing,positioning,fit}

\newtheorem{theorem}{Theorem}
\newtheorem{proposition}{Proposition}

\title{Multi-Stakeholder LLM Alignment: \\
Decomposing Estimation from Aggregation}

\author{
  \textbf{Lulu Zheng\textsuperscript{1}},
  \textbf{Wenjin Yang\textsuperscript{1}},
  \textbf{Xiangwen Zhang\textsuperscript{1}},
  \textbf{Rong Yin\textsuperscript{2}},
\\
  \textbf{Yulan Hu\textsuperscript{1}}\thanks{~Corresponding author.},
  \textbf{Zheng Pan\textsuperscript{1}},
  \textbf{Xin Li\textsuperscript{1}}
\\
\\
  \textsuperscript{1}AMAP, Alibaba Group \quad
  \textsuperscript{2}Beihang University
\\
  {\small         
  \texttt{\{zll522441,yangwenjin.ywj,xwzhang,huyulan,panzheng.pan,beilai.bl\}@alibaba-inc.com, yinrong@buaa.edu.cn}
  } \\  
}

\begin{document}
\maketitle

\begin{abstract}
Multi-stakeholder tasks require one output to satisfy users with conflicting preferences. Holistic LLM judges conflate utility estimation and utility aggregation, yielding unstable implicit weights. We show empirically and theoretically that this aggregation-specific \emph{weighting noise} can create large score shifts when stakeholder satisfaction is dispersed; in our experiments, these weight-induced shifts also increase with stakeholder count. We propose \textsc{DecompR}: counterfactual-calibrated weights are fixed from query structure before candidate scoring, while per-role utilities are estimated independently, removing candidate-dependent weight drift and reducing estimation noise.\end{abstract}

\section{Introduction}
\label{sec:intro}

Large language models (LLMs) have achieved remarkable success in code generation, mathematical reasoning, and open-ended dialogue \citep{openai2023gpt4, qwen2025qwen3, guo2025deepseek}. Much of this progress, however, targets scenarios where the agent serves a single user whose preferences define a clear optimization objective. Many real tasks---group travel planning, multi-party negotiation, committee resource allocation---instead require one agent to produce a single joint response for multiple stakeholders whose preferences may conflict.

Aligning LLMs for such tasks often relies on reinforcement learning \citep{ouyang2022training, shao2024deepseekmath, guo2025deepseek}, which requires a reward signal specifying what counts as a good joint response. Since these tasks are open-ended and lack verifiable ground truth, the reward is typically supplied by LLM-as-Judge \citep{zheng2023judging} or learned reward models \citep{ouyang2022training}. LLM judges are known to exhibit evaluation biases and instability even outside multi-stakeholder settings, including position bias in pairwise comparisons \citep{zheng2023judging, wang2023large}, adversarial vulnerabilities in absolute scoring \citep{raina2024llm}, and intra-rater score inconsistency across repeated runs \citep{haldar2025rating}. Multi-stakeholder evaluation compounds this problem: the judge must perform both \emph{estimation} (how well is each stakeholder served?) and \emph{aggregation} (how should competing utilities be weighted?) in a single pass, thereby producing implicitly assigned aggregation weights that can shift across calls.

As Figure~\ref{fig:problem} illustrates, repeated evaluations of the same plan can instantiate different implicit weights and produce inconsistent rewards. We refer to this as \emph{weighting noise}: an aggregation-specific error in which the judge's implicit stakeholder weights shift across evaluations, reweighting utilities. Its score impact is amplified when stakeholder satisfactions are dispersed and may increase with stakeholder count (\S\ref{sec:empirical}). Our theory shows how such shifts can corrupt within-group response ranking and mislead GRPO-style policy gradients (\S\ref{sec:theory}). We address this with \textsc{DecompR}, whose key mechanism is \emph{counterfactual-calibrated aggregation}: each stakeholder's weight is fixed from query structure as a proxy for expected sacrifice, while per-role grounded evaluation estimates utilities independently.

In summary, we formalize the multi-stakeholder alignment problem, show empirically and theoretically that LLM-judge reward inconsistency in this setting arises in part from aggregation-specific weighting noise, and propose \textsc{DecompR} to address this problem through counterfactual-calibrated estimation--aggregation separation.

\begin{figure*}[!t]
\centering
\includegraphics[width=\textwidth]{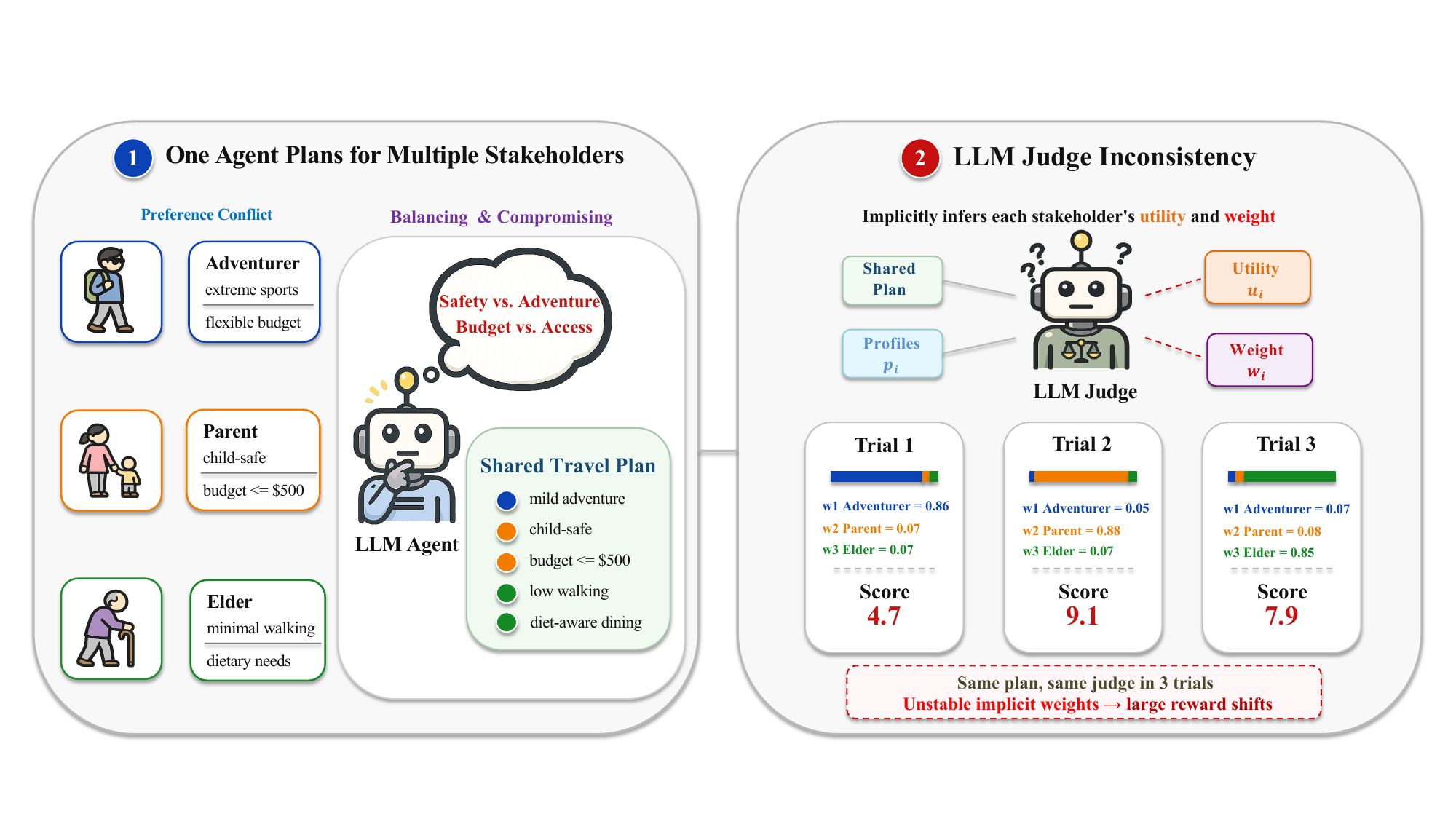}
\caption{(Left) A single LLM agent synthesizes conflicting stakeholder preferences into one shared plan. (Right) A holistic LLM judge implicitly assigns per-stakeholder utilities $u_i$ and weights $w_i$; repeated evaluations of the same plan yield different implicit weights and thus different scalar rewards.}
\label{fig:problem}
\end{figure*}

\section{Related Work}
\label{sec:related}

\subsection{Reward Evaluation for Open-Ended Tasks}

LLM-as-Judge \citep{zheng2023judging} enables reward estimation without verifiable ground truth, but exhibits well-documented inconsistencies: position bias \citep{wang2023large}, verbosity bias \citep{saito2023verbosity}, self-preference bias \citep{panickssery2024llm}, and intra-rater instability \citep{haldar2025rating}. Structured approaches improve reliability in two ways: explicit scoring criteria---rubrics \citep{hashemi2024llmrubric}, checklists \citep{lee2025checkeval}, and chain-of-thought scoring \citep{liu2023geval}---decompose evaluation into \emph{quality dimensions}, while trained evaluators---fine-tuned judges or generative reward models \citep{zhu2025judgelm, zhang2024generative}---replace direct prompting with supervised scoring models. These methods primarily address \emph{estimation} noise: making a specified scoring criterion or learned evaluator more reliable. Multi-stakeholder rewards add a separate \emph{aggregation} problem: the judge must decide how to weigh competing stakeholders, a trade-off these methods leave implicit.

\subsection{Multi-Objective Alignment}

Multi-objective alignment decomposes reward into predefined quality dimensions or controllable attributes (e.g., helpfulness, safety) and navigates their trade-offs via multi-attribute feedback \citep{dong2023steerlm, wang2024helpsteer2}, multi-objective reward models \citep{wang2024interpretable, wang2024arithmetic}, reward-conditioned generation \citep{yang2024rewards}, or policy merging \citep{rame2024rewarded, zhou2024modpo, jang2023personalized}. Our setting differs in that the aggregation axes are \emph{people}---varying in number, identity, and constraints per query---whose preferences may directly conflict. This requires query-specific aggregation and creates the weighting-noise problem we formalize in \S\ref{sec:theory}.

\subsection{GRPO and Reward Noise}

GRPO \citep{shao2024deepseekmath, guo2025deepseek} computes advantages via group-relative normalization, making gradient direction depend entirely on relative reward quality within each group. Recent work shows that group-relative advantages can be biased by prompt difficulty \citep{yang2026grpo_biased}, that noisy or corrupted rewards require explicit correction to avoid biased gradients \citep{mansouri2025noise_grpo}, and that multi-objective GRPO is vulnerable to reward hacking by disproportionately optimizing higher-variance reward components \citep{ichihara2025mogrpo}. These works study how GRPO responds to noise once a reward signal is given. We ask a prior question: whether the LLM-judge signal itself is reliable for multi-stakeholder tasks, why its variance arises, and how to reduce it before it enters GRPO.

\begin{table*}[t]
\centering
\small
\begin{tabular}{@{}ll cc cc cc cc cc@{}}
\toprule
& & \multicolumn{2}{c}{$n{=}2$} & \multicolumn{2}{c}{$n{=}3$} & \multicolumn{2}{c}{$n{=}5$} & \multicolumn{2}{c}{$n{=}8$} & & \\
\cmidrule(lr){3-4} \cmidrule(lr){5-6} \cmidrule(lr){7-8} \cmidrule(lr){9-10}
\textbf{Judge} & \textbf{Quality} & $\sigma^2_{\text{Rep}}$ & $\sigma^2_{\text{Sem}}$ & $\sigma^2_{\text{Rep}}$ & $\sigma^2_{\text{Sem}}$ & $\sigma^2_{\text{Rep}}$ & $\sigma^2_{\text{Sem}}$ & $\sigma^2_{\text{Rep}}$ & $\sigma^2_{\text{Sem}}$ & \textbf{Growth} & \textbf{S/R} \\
\midrule
\multirow{3}{*}{Qwen3.5-Plus}
 & Low     & 0.20 & 0.52 & 0.18 & 0.26 & 0.21 & 0.30 & 0.16 & 0.21 & 0.40 & 1.29 \\
 & Medium  & 0.15 & 0.59 & 0.63 & 0.68 & 0.67 & 0.73 & 0.25 & 0.46 & \textbf{0.78} & 1.87 \\
 & High    & 0.18 & 0.27 & 0.25 & 0.30 & 0.28 & 0.57 & 0.19 & 0.46 & \textbf{1.71} & 2.43 \\
\midrule
\multirow{3}{*}{Gemini-3-Flash}
 & Low     & 0.00 & 0.93 & 0.00 & 0.12 & 0.00 & 0.18 & 0.00 & 0.15 & 0.16 & $\infty$ \\
 & Medium  & 0.00 & 0.25 & 0.00 & 0.81 & 0.00 & 1.72 & 0.00 & 1.67 & \textbf{6.67} & $\infty$ \\
 & High    & 0.00 & 0.64 & 0.00 & 0.64 & 0.00 & 1.49 & 0.00 & 1.46 & \textbf{2.27} & $\infty$ \\
\midrule
\multirow{3}{*}{DeepSeek-V4-Flash}
 & Low     & 0.43 & 1.14 & 0.57 & 0.45 & 0.37 & 0.51 & 0.26 & 0.39 & 0.34 & 1.54 \\
 & Medium  & 0.47 & 1.10 & 1.06 & 1.52 & 1.77 & 1.46 & 1.31 & 1.32 & \textbf{1.20} & 1.01 \\
 & High    & 0.50 & 0.68 & 1.14 & 1.09 & 0.91 & 1.72 & 0.49 & 1.22 & \textbf{1.81} & 2.49 \\
\midrule
\multicolumn{2}{@{}l}{\textit{All (Med+High)}} & 0.22 & 0.59 & 0.53 & 0.84 & 0.62 & 1.28 & 0.42 & 1.10 & \textbf{1.87} & 2.64 \\
\multicolumn{2}{@{}l}{\textit{All (Low)}} & 0.21 & 0.86 & 0.25 & 0.28 & 0.19 & 0.33 & 0.14 & 0.25 & \textbf{0.29} & 1.80 \\
\bottomrule
\end{tabular}
\caption{$\sigma^2_{\text{Rep}}$ = exact-repeat variance; $\sigma^2_{\text{Sem}}$ = presentation-induced reward inconsistency, measured as score variance over semantics-preserving variants of the same response; Growth = $\sigma^2_{\text{Sem}}$\textsubscript{$n$=8}/$\sigma^2_{\text{Sem}}$\textsubscript{$n$=2}; S/R = $\sigma^2_{\text{Sem}}$/$\sigma^2_{\text{Rep}}$ at $n{=}8$. Gemini-3-Flash: $\sigma^2_{\text{Rep}}{=}0$ at $T{=}0$.}
\label{tab:scaling}
\end{table*}

\section{Empirical Analysis: Multi-Stakeholder Reward Consistency}
\label{sec:empirical}

Multi-stakeholder planning lacks verifiable ground truth, so RL training must rely on LLM judges or reward models to provide the learning signal. We empirically investigate whether current LLM judges can provide a consistent reward signal for this setting---where a single plan must reconcile heterogeneous and sometimes conflicting stakeholder preferences.

\paragraph{Setup.} We construct 60 seed responses from real multi-stakeholder travel-planning queries, spanning stakeholder counts $n\in\{2,3,5,8\}$ and three quality levels (high/medium/low). For each seed we generate $55$ \emph{semantics-preserving variants} (11 variant families $\times$ 5 versions) via rule-based transforms (order permutation, formatting controls) and LLM-based rewrites (causal-direction flipping, evidence relocation, paraphrase, topic repositioning), all preserving itinerary content, constraint satisfaction, and plan quality. A human evaluator would assign these variants the same score; we define $\sigma^2_{\text{Sem}}$ as the score variance across variants for the same seed response. We additionally collect 5 exact repeats of the unchanged input as a noise baseline, yielding $\sigma^2_{\text{Rep}}$, the judge's intrinsic stochasticity. Their ratio $\sigma^2_{\text{Sem}}/\sigma^2_{\text{Rep}} > 1$ indicates that surface presentation adds variance beyond intrinsic noise. All evaluations use greedy decoding ($T{=}0$) across three state-of-the-art LLM judges: Qwen3.5-Plus \citep{qwen2026qwen35}, Gemini-3-Flash \citep{google2025gemini3flash}, and DeepSeek-V4-Flash \citep{deepseekai2026deepseekv4}. Full details are in Appendix~\ref{app:setup}.

\subsection{Score Instability Scales with Stakeholder Count}
\label{sec:exp_scaling}

Table~\ref{tab:scaling} shows that \textbf{score variance scales with stakeholder count} for medium- and high-quality responses: pooled Growth${=}1.87$, with Gemini reaching $6.67{\times}$. At $n{=}8$, the Sem/Rep ratio reaches 2.64 for medium+high responses: merely changing how stakeholders are discussed---their ordering, where evidence is placed---causes $2.6{\times}$ more score variation than the judge's own stochasticity. Low-quality responses show no such scaling (Growth${=}0.29$), likely due to a floor effect: scores cluster near the bottom of the scale ($\mu{\approx}2.5$--3.2), leaving the judge little room to vary. This is especially concerning for GRPO, which relies on within-group reward comparisons: unstable scores for medium- and high-quality responses---the candidates GRPO must distinguish among---can flip which responses are reinforced.

Not all variant types contribute equally: discourse changes tied to stakeholder presentation (causal framing, evidence placement, coverage order) produce the largest variance growth as $n$ increases, while low-level formatting changes show little growth (Appendix~\ref{app:variant_breakdown}). The instability is thus stakeholder-specific, not generic surface noise.

\subsection{The Stability--Expressiveness Trade-off in Structured Scoring}
\label{sec:variance_reduction}

We next test whether structured evaluation protocols can reduce this presentation-induced variance. We compare direct holistic scoring with three structured protocols (Appendix~\ref{app:rubric}): \textbf{Rubric} scoring over fixed dimensions with fixed weights \citep{hashemi2024llmrubric}, \textbf{Checklist} with binary verification and programmatic aggregation \citep{lee2025checkeval}, and \textbf{Decomposed} scoring in two variants---\emph{adaptive}, where the judge assigns per-stakeholder weights $w_i$ and satisfactions $\hat{u}_i$, and \emph{uniform}, a diagnostic recomputation using $w_i{=}1/n$ with the same $\hat{u}_i$. All use Gemini-3-Flash at $T{=}0$ ($\sigma^2_{\text{Rep}}{=}0$) so that all observed variance is purely presentation-induced. We focus on medium- and high-quality responses---the quality range most relevant to GRPO, where the policy must distinguish among competitive candidates.

\begin{table}[t]
\centering
\resizebox{\columnwidth}{!}{%
\begin{tabular}{@{}l cccc c@{}}
\toprule
\textbf{Method} & $n{=}2$ & $n{=}3$ & $n{=}5$ & $n{=}8$ & \textbf{Growth} \\
\midrule
Direct      & 0.45 & 0.72 & 1.61 & 1.57 & 3.50 \\
Rubric      & 0.70 & 0.65 & 0.79 & 0.82 & 1.18 \\
Checklist   & 1.10 & 1.06 & 1.47 & 2.22 & 2.01 \\
Decomposed (adaptive) & 0.61 & 0.87 & 0.88 & 0.96 & 1.58 \\
Decomposed (uniform)  & 0.60 & 0.81 & 0.83 & 0.72 & 1.20 \\
\bottomrule
\end{tabular}}
\caption{$\sigma^2_{\text{Sem}}$ by evaluation method (Gemini-3-Flash, $T{=}0$, medium+high only).}
\label{tab:structured}
\end{table}

The key insight is that structure helps reduce presentation-induced variance, but online judge-adaptive stakeholder weighting reintroduces aggregation instability. Rubric (Growth 1.18) and Decomposed with fixed uniform weights (1.20) are most stable, while the controlled decomposed comparison shows that judge-assigned adaptive weights increase Growth from 1.20 to 1.58 using the same $\hat{u}_i$. These stable baselines are not sufficient as final rewards: Rubric aggregates generic quality dimensions rather than stakeholder utilities, and uniform weights ignore query-specific constraint asymmetry. This motivates a two-stage design: estimate per-stakeholder satisfaction online, but anchor non-uniform aggregation weights outside the scoring call. We next quantify the resulting weight-induced score shifts.

\subsection{Weight Drift: The Cost of Adaptive Aggregation}
\label{sec:exp_implicit_weights}

We isolate the instability introduced by online stakeholder-weight assignment. For each variant group, we normalize the resulting weight-induced score shift by the within-group final-score standard deviation $\sigma$, making it comparable to the reward differences used for GRPO ranking.

Decomposed (adaptive) asks the judge to output per-stakeholder weight $w_i$ and satisfaction $\hat{u}_i$, scoring $\sum_i w_i \hat{u}_i$. Since the query is unchanged across variants, the target aggregation weights under any fixed reward definition should remain fixed. We use the variant-mean $\bar{w}_i$ as reference and measure the score shift from weight drift:
\begin{equation}
\Delta R_w=\sum_i (w_i-\bar{w}_i)\hat{u}_i .
\end{equation}

\begin{table}[t]
\centering
\resizebox{\columnwidth}{!}{%
\begin{tabular}{@{}c cc ccc c@{}}
\toprule
$n$ & $w$ CV & $\hat{u}$ CV & Mean$\frac{|\Delta R_w|}{\sigma}$ & P95$\frac{|\Delta R_w|}{\sigma}$ & $\frac{\text{Shift Range}}{\sigma}$ & $w$ Var\% \\
\midrule
2  & 0.050 & 0.236 & 2.6\% & 10.9\% & 18.2\% &  1.9 \\
3  & 0.066 & 0.143 & 9.4\% & 27.9\% & 51.4\% &  6.0 \\
5  & 0.067 & 0.201 & 10.7\% & 37.7\% & 61.4\% &  7.4 \\
8  & 0.106 & 0.183 & 14.4\% & 37.9\% & 81.2\% &  5.3 \\
\bottomrule
\end{tabular}}
\caption{Weight drift as a fraction of within-group score spread (Gemini-3-Flash, $T{=}0$). $\sigma$: mean within-group standard deviation of the final score, computed per variant group then averaged over groups with the same $n$. CV${=}\sigma/\mu$. $w$~Var\%: fraction of total score variance explained by weight drift.}
\label{tab:weight_drift}
\end{table}

\paragraph{Multiplier effect.}
Table~\ref{tab:weight_drift} normalizes weight-induced score shifts by the within-group score standard deviation $\sigma$, directly showing how much of the advantage signal weight drift can corrupt. Weight drift contributes modestly to average variance ($w$~Var\%${=}2$--$7\%$), yet at $n{=}8$ it typically shifts a response's score by $14\%$ of $\sigma$ (mean), with $38\%$ at the 95th percentile and up to $81\%$ in the worst case (shift range). The mechanism is multiplicative: even small weight fluctuations ($w$~CV${\le}0.11$) are multiplied by cross-stakeholder satisfaction differences, so their score impact grows with the spread of $\hat{u}_i$ rather than with weight variance alone. Appendix~Table~\ref{tab:multiplier_example} gives the concrete breakdown for one such case.

These magnitudes imply that weight noise alone can flip the sign of GRPO advantages for response pairs whose reward gaps lie within a one-standard-deviation band. The next section analyzes how weight noise enters reward variance and which factors amplify its effect.

\section{Theoretical Analysis: Decomposing Judge Variance}
\label{sec:theory}

Section~\ref{sec:empirical} documented \emph{what} happens: LLM-as-a-Judge scoring is inconsistent despite structured methods. This section analyzes \emph{why} and \emph{how it affects GRPO}: \S\ref{sec:formulation}--\ref{sec:scaling} trace the inconsistency to the judge's unstable implicit weighting of stakeholders---an \emph{aggregation error} amplified by utility dispersion---and \S\ref{sec:snr_analysis} shows when this error corrupts GRPO's advantage signal. The semantics-preserving variants in Section~\ref{sec:empirical} hold $R^*(y)$ fixed, so their score spread estimates the judge-error scale $\sigma_\xi^2$; GRPO is affected when this same-response noise is large relative to the true reward-gap spread within a sampled group.
\subsection{Variance Decomposition}
\label{sec:formulation}

Let $u_i^*(y)\in[0,1]$\footnote{The empirical utility heads report scores on a $[1,10]$ scale; we linearly normalize to $[0,1]$ for this theoretical notation. Variance and SNR scale by the corresponding constant factor.} be stakeholder $i$'s latent satisfaction and $w_i^*$ fixed target weights with $\sum_i w_i^*=1$. We analyze the ideal cardinal welfare reward \citep{harsanyi1955cardinal}:
\begin{equation}
    R^*(y) = \sum_{i=1}^{n} w_i^* \cdot u_i^*(y)
    \label{eq:true_reward}
\end{equation}

We model a holistic judge as implicitly estimating both utilities and weights:
\begin{equation*}
    \hat{R}(y)=\sum_i (w_i^*+\eta_i)(u_i^*(y)+\delta_i)+\epsilon ,
\end{equation*}
where $\delta_i$ is utility-estimation noise, $\eta_i$ is implicit weight drift, and $\epsilon$ is residual score-level noise. Normalization gives $\sum_i\eta_i=0$.

\begin{theorem}[Three-Term Variance Decomposition]
\label{thm:decomposition}
Assume zero-mean noise ($E[\delta_i] = E[\eta_i] = E[\epsilon] = 0$), independence across the three noise groups ($\boldsymbol{\delta}\perp\boldsymbol{\eta}\perp\epsilon$), and independence within $\{\delta_i\}$; the $\eta_i$ may be correlated. Then:
\begin{align}
    \mathrm{Var}[\hat{R}(y)]
    &= \underbrace{\sum_{i} (w_i^*)^2 \sigma_{\delta_i}^2}_{\text{Term I: Estimation}}
     + \underbrace{\mathrm{Var}\!\left[\textstyle\sum_i \eta_i u_i^*\right]}_{\text{Term II: Aggregation}} \nonumber \\
    &\quad + \underbrace{\sigma_\epsilon^2}_{\text{Term III: Shared}}
     + \underbrace{\sum_i \sigma_{\eta_i}^2 \sigma_{\delta_i}^2}_{C_{\text{cross}}}
    \label{eq:decomposition}
\end{align}
(Proof in Appendix~\ref{app:proof_decomposition}.)
\end{theorem}

\paragraph{Insight.} Term~I is reducible via structured evaluation (\S\ref{sec:variance_reduction}). Term~II captures an aggregation-specific error: unless the aggregation rule is fixed, zero-sum weight drift $\eta_i$ is multiplied by utility differences, so its score impact is amplified when stakeholder satisfactions are dispersed (as measured in Table~\ref{tab:weight_drift}). This is the key distinction from single-user noise: multi-stakeholder reward variance includes a component tied to weight noise that is absent in single-reward settings.

\subsection{Conditional Scaling and Irreducibility}
\label{sec:scaling}

We next characterize how Term~II (weighting noise) depends jointly on weight-drift variance, cross-stakeholder utility dispersion, and stakeholder count.

\begin{proposition}[Conditional Scaling of Term~II]
\label{prop:scaling}
For $n\ge2$, assume exchangeable zero-sum weight noise: $\sum_i\eta_i=0$, $\mathrm{Var}[\eta_i]=\sigma_\eta^2(n)$ for all $i$, and a common off-diagonal covariance. Then:
\begin{align*}
    \text{Term~II}
    &= \frac{n}{n-1}\,\sigma_\eta^2(n)
       \sum_{i=1}^n (u_i^* - \bar{u})^2 \\
    &= \frac{n^2}{n-1}\,\sigma_\eta^2(n) \mathrm{Var}_i[u_i^*].
\end{align*}
where $\bar{u}=\frac{1}{n}\sum_i u_i^*$ and $\mathrm{Var}_i[u_i^*]\triangleq \frac{1}{n}\sum_i(u_i^*-\bar{u})^2$ is across-stakeholder variance for a fixed response. (Proof in Appendix~\ref{app:proof_scaling}.)
\end{proposition}

\paragraph{Insight.} Proposition~\ref{prop:scaling} shows that Term~II has three necessary factors: weight instability $\sigma_\eta^2(n)$, utility dispersion $\mathrm{Var}_i[u_i^*]$, and the stakeholder-count factor $n^2/(n-1)$. If the judge's implicit weights do not drift, Term~II vanishes. If all stakeholders have equal utility ($\mathrm{Var}_i[u_i^*]{=}0$), zero-sum weight drift is score-invariant. When both weight drift and utility heterogeneity are present, shifting weight toward relatively satisfied or underserved stakeholders changes the scalar reward, with the magnitude scaled by $n^2/(n-1)$. Fixed stakeholder weights remove this channel by setting $\eta_i{=}0$. Structured evaluation can reduce utility-estimation noise (Term~I), but it leaves aggregation noise (Term~II) whenever the judge still chooses the weights online. Empirically, the weight-induced shift range grows from 0.15 ($n{=}2$) to 0.61 ($n{=}8$) in Table~\ref{tab:weight_drift}. Appendix~\ref{app:proof_scaling} derives the symmetric scaling formula and shows how the same mechanism extends to general weight-drift covariance.

\subsection{When Weighting Noise Corrupts GRPO}
\label{sec:snr_analysis}

GRPO updates are driven by each candidate's reward gap from the group mean. In the noise-free target, candidate $j$ should be updated according to its true reward gap $\Delta_j^* = R_j^*-\bar{R}^*$ (or the standardized gap $\Delta_j^*/\mathrm{sd}_k(R_k^*)$). In practice, the judge observes $\hat{R}_j=R_j^*+\xi_j$, so GRPO uses the noisy gap
\begin{equation}
    \hat{\Delta}_j
    = \hat{R}_j-\bar{\hat{R}}
    = \Delta_j^* + (\xi_j-\bar{\xi}).
\end{equation}
Here $\Delta_j^*$ is the \textbf{true reward gap} we want GRPO to use: the
within-group quality difference between candidate $j$ and the group mean.
The second term, $\xi_j-\bar{\xi}$, is the group-centered \textbf{random error}
induced by the LLM-judge inconsistency studied in Section~\ref{sec:empirical}:
even the same semantic response can receive different scores under
presentation changes. This error turns the true reward gap into the noisy
reward gap $\hat{\Delta}_j$ and can flip the advantage sign when it is
comparable to $\Delta_j^*$. We summarize this reliability as:
\begin{equation}
\mathrm{SNR} \triangleq \mathrm{Var}_j[\Delta_j^*]/\sigma_\xi^2.
\label{eq:snr_def}
\end{equation}
Thus SNR compares the true reward-gap signal against the judge-error scale
that creates the noisy reward gap. As $\sigma_\xi^2$ grows, this noise more
easily overwhelms the reward-gap signal that GRPO uses for ranking.

\begin{proposition}[Advantage Sign Correctness]
\label{prop:snr_sign}
Under independent Gaussian reward errors with variance $\sigma_\xi^2$, the probability that GRPO assigns the correct advantage sign to a response one standard deviation from the group mean is:
\begin{equation}
    P(\textrm{correct sign}) = \Phi\!\left(\frac{\sqrt{\mathrm{SNR}}}{\sqrt{1-1/G}}\right)
    \label{eq:p_correct}
\end{equation}
\end{proposition}

Combining this with Proposition~\ref{prop:scaling}, aggregation noise harms GRPO when Term~II is non-negligible: it increases $\sigma_\xi^2$ most in high-dispersion stakeholder groups, reducing SNR exactly where trade-offs are hardest. In the conditional limit where Term~II is the leading noise source and $\sigma_\eta^2(n)\mathrm{Var}_i[u_i^*]$ does not shrink with $n$, $\mathrm{SNR}(n)=O(1/n)$; substituting into Eq.~\ref{eq:p_correct} gives:
\begin{equation}
    P(\textrm{correct sign}) = \frac{1}{2}+O(1/\sqrt{n}).
    \label{eq:sign_degradation}
\end{equation}
Thus, weighting noise can turn group-relative advantages into a noisy proxy for true quality, pushing GRPO updates toward judge-induced aggregation shifts rather than the intended reward. This motivates \textsc{DecompR}: fix aggregation outside online judge calls so the online weight-drift channel is removed by construction. The derivation and numerical thresholds are given in Appendix~\ref{app:snr_derivation}.

\section{\textsc{DecompR}: Counterfactual-Calibrated Decomposed Reward}
\label{sec:method}

\begin{figure*}[t]
\centering
\begin{tikzpicture}[
    scale=1.12,
    transform shape,
    >=stealth,
    base/.style={draw, rounded corners=2pt, align=center, font=\scriptsize, inner sep=4pt},
    input/.style={base, fill=blue!7, minimum width=2.65cm, minimum height=0.48cm},
    module/.style={base, fill=#1, text width=4.25cm, minimum height=1.7cm},
    module/.default={gray!8},
    rewardbox/.style={base, fill=orange!10, minimum width=5.55cm, minimum height=0.68cm},
    arr/.style={->, thick, gray!65},
    warr/.style={->, thick, violet!65},
    earr/.style={->, thick, green!45!black},
    note/.style={font=\tiny, align=center},
    framebox/.style={draw=blue!35!gray, dashed, rounded corners=7pt, fill=blue!3, line width=0.65pt},
]

\node[framebox, minimum width=11.25cm, minimum height=5.4cm] at (5.35, 0.35) {};

\node[input] (query) at (2.35, 2.35)
  {query + profiles\\$q,p_{1:n}$};
\node[input] (candidate) at (8.35, 2.35)
  {candidate response\\$y$};

\node[module=violet!8, draw=violet!60, minimum height=1.7cm] (weights) at (2.35, 0.72)
  {\textbf{1. Offline Counterfactual Weights}\\[2pt]
   $d_i=\sum_{c\in C_i^{\mathrm{hard}}}\alpha(c)$\\[-1pt]
   $\quad+\gamma\sum_{c\in C_i^{\mathrm{soft}}}\alpha(c)+\beta s_i^{\mathrm{conflict}}$\\[2pt]
   $w_i=\mathrm{softmax}(d_i/\tau_w)$
   };

\node[module=green!8, draw=green!45!black, minimum height=1.7cm] (estimator) at (8.35, 0.72)
  {\textbf{2. Online Per-role Utility Estimation}\\[2pt]
   $\hat{u}_i(y) = \operatorname{Eval}_{\mathrm{grounded}}(p_i,y)$\\[2pt]
   {\tiny hard/soft checks + judge; role isolation}};

\node[rewardbox] (reward) at (5.35, -1.0)
  {$\hat{R}(y)=\displaystyle\sum_i w_i\hat{u}_i(y)$};
\node[note, font=\scriptsize, text=gray!65!black] (out) at (5.35, -1.88)
  {scalar reward for GRPO};

\draw[warr] (query) -- (weights);
\draw[earr] (candidate) -- (estimator);
\draw[earr] (query.east) -- (estimator.north west);
\draw[warr] (weights.south) |- node[pos=0.25, left, note, text=violet!70] {$w_i$} (reward.west);
\draw[earr] (estimator.south) |- node[pos=0.25, right, note, text=green!35!black] {$\hat{u}_i$} (reward.east);
\draw[arr] (reward.south) -- (out.north);

\end{tikzpicture}
\caption{Overview of \textsc{DecompR}: offline query-level weights remove candidate-dependent weight drift, online per-role utilities reduce estimation error, and fixed aggregation produces the scalar reward for GRPO.}
\label{fig:method_overview}
\end{figure*}
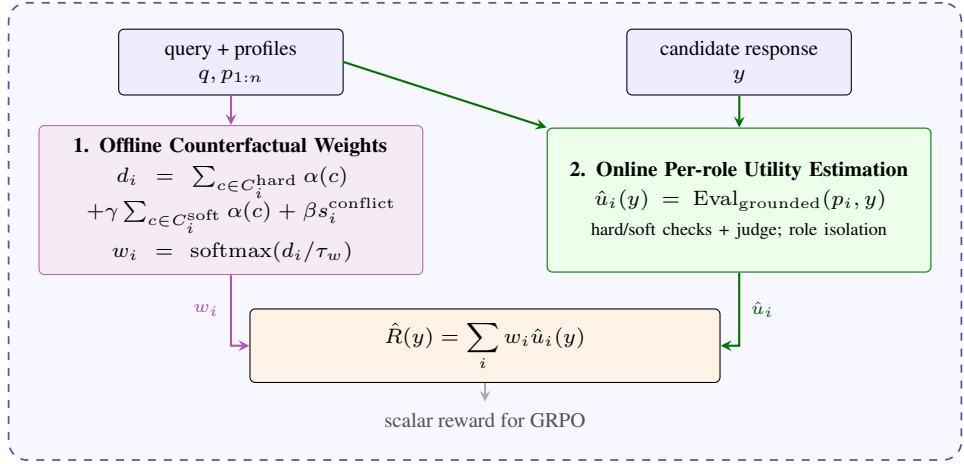

\S\ref{sec:theory} shows that one source of holistic reward instability is that the judge's implicit trade-off rule can change across candidate responses. Our method, \textsc{DecompR}, addresses this with counterfactual-calibrated offline query-level aggregation and online grounded per-role utility estimation (Figure~\ref{fig:method_overview}).

\subsection{Counterfactual-Calibrated Aggregation (Removing Online Weight Drift)}
\label{sec:aggregation}

We remove online weight drift by fixing aggregation weights once per query. The calibration follows a counterfactual-sacrifice intuition from bargaining theory \citep{kalai1975other}: stakeholders who lose more from joining the shared task, relative to their solo ideal, should receive higher aggregation priority. Conceptually, with solo plan $y_i^{\mathrm{solo}}$ and joint reference plan $y^{\mathrm{joint}}$:
\begin{equation}
    \Delta_i^{\mathrm{cf}}
    = U_i(y_i^{\mathrm{solo}}) - U_i(y^{\mathrm{joint}}),
    \label{eq:counterfactual_sacrifice}
\end{equation}
These weights may not be uniquely correct, but they incorporate stakeholder fairness into the offline fixed aggregation rule before any candidate response is scored.

In implementation, we do not explicitly compute the counterfactual utility gap in Eq.~\ref{eq:counterfactual_sacrifice}, since this would require costly solo and joint planning estimates. Instead, we use demand difficulty as a lightweight proxy: stakeholders whose requirements are harder to satisfy receive larger weights. This follows the same intuition as Eq.~\ref{eq:counterfactual_sacrifice}: hard-to-satisfy needs should not be silently sacrificed. Concretely, an offline parser estimates difficulty from restrictive hard constraints, dense soft preferences, and direct conflicts; fixed rules then compute:
\begin{equation}
    d_i =
    \sum_{c \in C_i^{\mathrm{hard}}}\alpha(c)
    + \gamma \sum_{c \in C_i^{\mathrm{soft}}}\alpha(c)
    + \beta s_i^{\mathrm{conflict}}
    \label{eq:difficulty}
\end{equation}
where $C_i^{\mathrm{hard}}$ and $C_i^{\mathrm{soft}}$ are parsed constraints, $\alpha(c)$ is a fixed restrictiveness score, $s_i^{\mathrm{conflict}}$ summarizes pairwise tensions involving stakeholder $i$, and $\gamma,\beta\in(0,1)$ are fixed discounts. All quantities are computed from the query and profiles only, before any candidate response is scored. The resulting $d_i$ is a deterministic proxy for $\mathbb{E}[\Delta_i^{\mathrm{cf}}\mid q,p_{1:n}]$, not a candidate score. We then normalize these query-level difficulty scores into weights:
\begin{equation}
    w_i = \frac{\exp(d_i / \tau_w)}{\sum_{j=1}^{n} \exp(d_j / \tau_w)}
    \label{eq:weights}
\end{equation}
where $\tau_w$ controls differentiation ($\tau_w \to \infty$ recovers uniform; $\tau_w \to 0$ concentrates on the hardest stakeholder). Since $d_i$ and $w_i$ are fixed for a given query, candidate-dependent weight drift is removed even if the calibrated weights are only approximate.

\paragraph{Illustrative example.} Table~\ref{tab:weight_example} shows how calibrated weights expose failures hidden by uniform aggregation: the hardest stakeholder receives more weight, so underserving that stakeholder lowers the reward.

\begin{table}[t]
\centering
\small
\begin{tabular}{@{}lccccc@{}}
\toprule
\textbf{Stakeholder} & $d_i$ & $w_i^{\text{u}}$ & $w_i^{\text{c}}$ & $\hat{u}_i$ & \textbf{Sacrifice} \\
\midrule
A (4 hard + 2 soft) & 5.0 & 0.33 & \textbf{0.67} & 0.40 & 50\% \\
B (0 hard + 2 soft) & 1.0 & 0.33 & 0.09 & 0.95 & 3\% \\
C (2 hard + 2 soft) & 3.0 & 0.33 & 0.24 & 0.80 & 11\% \\
\midrule
\multicolumn{3}{@{}l}{$\hat{R}_{\text{uniform}} = 0.72$} & \multicolumn{3}{l}{$\hat{R}_{\text{calib}} = \textbf{0.55}$} \\
\bottomrule
\end{tabular}
\caption{Toy three-traveler calibration ($\gamma{=}0.5$, $\tau_w{=}2$). Values are illustrative; implementation uses the query-level difficulty proxy $d_i$ rather than computing solo optima online.}
\label{tab:weight_example}
\end{table}

\subsection{Online Per-Role Grounded Estimation (Reducing Term~I)}
\label{sec:estimation}

Given a candidate response $y$, \textsc{DecompR} estimates each stakeholder's utility independently:
\begin{equation}
    \hat{u}_i(y) = \operatorname{Eval}_{\mathrm{grounded}}(p_i,y).
    \label{eq:utility_estimator}
\end{equation}
The evaluator only asks whether $y$ satisfies profile $p_i$, rather than asking one judge to balance all stakeholders at once. This per-role isolation is designed to keep the estimation task in the simpler $n{=}1$ regime, avoiding the context and attention load of holistic multi-stakeholder judging.

Within each role, signals are layered by decreasing determinism: tool-verified facts ($\sigma^2{=}0$), programmatic constraint checks ($\sigma^2{\approx}0$), LLM-judged soft preferences, and a residual subjective checklist with weight at most $20\%$. Thus most reward mass comes from deterministic or low-variance components.

\section{Experiments}
\label{sec:experiments}

We first verify that \textsc{DecompR}'s aggregation design reduces reward variance (\S\ref{sec:reward_consistency}), then evaluate whether this translates to better plan quality in end-to-end GRPO training (\S\ref{sec:e2e}).

\subsection{Reward Consistency}
\label{sec:reward_consistency}

Using the same setup as \S\ref{sec:variance_reduction}, we recompute decomposed scores under different fixed weight schemes while holding $\hat{u}_i$ unchanged, isolating the effect of aggregation weights on reward stability.

\begin{table}[t]
\centering
\small
\resizebox{\columnwidth}{!}{%
\begin{tabular}{@{}l cccc c@{}}
\toprule
\textbf{Aggregation} & $n{=}2$ & $n{=}3$ & $n{=}5$ & $n{=}8$ & \textbf{Growth} \\
\midrule
Adaptive (judge $w_i$) & 0.61 & 0.87 & 0.88 & 0.96 & 1.58 \\
Uniform ($w_i{=}1/n$)  & 0.60 & 0.81 & 0.83 & 0.72 & 1.20 \\
\textsc{DecompR} ($\tau_w{=}3$) & 0.63 & 0.80 & 0.92 & 0.78 & 1.24 \\
\textsc{DecompR} ($\tau_w{=}5$) & 0.64 & 0.79 & 0.88 & 0.75 & 1.18 \\
\textsc{DecompR} ($\tau_w{=}8$) & 0.65 & 0.79 & 0.86 & 0.73 & 1.14 \\
\bottomrule
\end{tabular}
}
\caption{$\sigma^2_{\text{Sem}}$ under different fixed aggregation weights (same $\hat{u}_i$ from decomposed evaluation). Growth = $\sigma^2_{\text{Sem}}$\textsubscript{$n$=8}/$\sigma^2_{\text{Sem}}$\textsubscript{$n$=2}. $\tau_w$ controls weight differentiation (Eq.~\ref{eq:weights}).}
\label{tab:reward_consistency}
\end{table}

Table~\ref{tab:reward_consistency} shows that any per-query fixed weight eliminates candidate-dependent weight drift, but the degree of $n$-scaling reduction depends on weight concentration $\sum_i w_i^2$. Adaptive weights (Growth${=}1.58$) allow drift across variants; uniform weights (Growth${=}1.20$) minimize $\sum_i w_i^2{=}1/n$ but lose query-specific differentiation. \textsc{DecompR} with $\tau_w{=}5$ achieves Growth${=}1.18$, close to uniform stability while preserving meaningful constraint-based prioritization; $\tau_w{=}8$ is slightly more stable (Growth${=}1.14$) but closer to uniform weighting. The temperature $\tau_w$ controls the stability--expressiveness trade-off: smaller $\tau_w$ concentrates weight on high-difficulty stakeholders (more expressive but amplifies their estimation noise); larger $\tau_w$ approaches uniform (more stable but loses stakeholder differentiation).

\subsection{End-to-End Training}
\label{sec:e2e}

We next evaluate whether the variance reduction translates to better downstream plan quality when used as the GRPO reward signal.

\paragraph{Setup.} We use Qwen3-30B-A3B-Thinking-2507 \citep{qwen2025qwen3} as the base policy model: after supervised fine-tuning on expert trajectories generated by Qwen3.5-Plus, it undergoes GRPO training on 6{,}000 queries with the four scoring methods from \S\ref{sec:variance_reduction} (Direct, Rubric, Checklist, \textsc{DecompR} with $\tau_w{=}5$) as reward functions. We additionally include the base model (without RL) as a reference. Best checkpoint is selected within 400 steps; evaluation uses MR-TravelBench \citep{cheng2026grouptravelbench} (650 held-out tasks $\times$ 3 trials). Additional training details are in Appendix~\ref{app:training_details}.

\begin{table*}[!t]
\centering
\small
\resizebox{\textwidth}{!}{%
\begin{tabular}{@{}l cccc c cccc c cccc@{}}
\toprule
& \multicolumn{4}{c}{\textbf{Group Utility} $\uparrow$} & & \multicolumn{4}{c}{\textbf{Group Fairness} $\uparrow$} & & \multicolumn{4}{c}{\textbf{Pref.\ Completeness} $\uparrow$} \\
\cmidrule(lr){2-5} \cmidrule(lr){7-10} \cmidrule(lr){12-15}
\textbf{Reward} & Easy & Med & Hard & \textbf{All} & & Easy & Med & Hard & \textbf{All} & & Easy & Med & Hard & \textbf{All} \\
\midrule
\textsc{DecompR} & \textbf{5.94} & \textbf{7.35} & \textbf{7.85} & \textbf{7.07} && 60.10 & \textbf{49.58} & \textbf{39.84} & \textbf{49.82} && \textbf{24} & \textbf{19} & 15 & \textbf{19} \\
Direct           & 5.34 & 6.68 & 7.42 & 6.50 && 58.52 & 45.20 & 35.82 & 46.42 && 22 & 17 & 15 & 18 \\
Rubric           & 5.49 & 6.30 & 7.14 & 6.31 && 59.76 & 44.69 & 35.86 & 46.61 && 22 & 17 & 14 & 18 \\
Checklist        & 5.21 & 6.25 & 7.38 & 6.28 && 58.14 & 44.82 & 36.51 & 46.23 && 21 & 17 & 15 & 17 \\
\midrule
Base Model       & 5.11 & 6.19 & 7.20 & 6.17 && \textbf{60.50} & 46.52 & 37.69 & 48.11 && 23 & 18 & 14 & 18 \\
\bottomrule
\end{tabular}}
\caption{MR-TravelBench results (650 tasks $\times$ 3 trials, Easy/Medium/Hard = 200/250/200 tasks). \textsc{DecompR} achieves the highest Group Utility (+8.8\% vs Direct, +12.0\% vs Rubric) and Group Fairness (+7.3\% vs Direct), with the advantage growing on Hard tasks (+5.8\% Group Utility, +11.2\% Group Fairness vs Direct). Base Model is the pre-RL reference.}
\label{tab:main_results}
\end{table*}

Table~\ref{tab:main_results} shows that \textsc{DecompR} achieves the best performance on all core metrics. All GRPO configurations improve Group Utility over the Base Model (6.17), confirming that RL training is effective; however, Direct, Rubric, and Checklist rewards all \emph{reduce} Group Fairness below the Base Model (48.11), suggesting that noisy reward signals may encourage GRPO to optimize utility at the expense of equity. Only \textsc{DecompR} improves both simultaneously (Group Utility${=}7.07$, Group Fairness${=}49.82$), consistent with the counterfactual-calibrated weighting that prevents the agent from satisfying only the easy-to-serve travelers. Among GRPO methods, Checklist performs worst on Group Utility (6.28) and Preference Completeness (17), consistent with its high presentation-induced variance (Growth${=}2.01$ in \S\ref{sec:variance_reduction}): binary checks that flip under surface changes produce a noisy reward that misleads GRPO.

The advantage \emph{increases with task difficulty}: on Hard tasks---where conflicts are more severe and satisfaction dispersion is expected to be higher---\textsc{DecompR} outperforms Direct by +5.8\% in utility and +11.2\% in fairness. This aligns with our theoretical analysis: the multiplier effect of weight drift is most damaging in high-dispersion scenarios (\S\ref{sec:theory}), and fixing aggregation outside the judge provides the largest benefit where it matters most.

\section{Conclusion}
\label{sec:conclusion}

This work identifies a central difficulty in multi-stakeholder LLM-judge evaluation: rewards become unreliable when a single judge call estimates stakeholder utilities and decides cross-stakeholder weights. Our theory and experiments show that this coupling creates \emph{weighting noise}, whose score impact grows with satisfaction dispersion and stakeholder count. Structured scoring reduces utility-estimation noise but leaves this online aggregation channel intact.

\textsc{DecompR} addresses this through \emph{estimation-aggregation separation}: online grounded estimation measures per-role utilities, while offline counterfactual calibration fixes auditable aggregation weights before candidate scoring.

More broadly, multi-stakeholder reward modeling remains underexplored relative to single-user settings. We hope this work encourages more systematic study of how LLM systems should evaluate, aggregate, and learn from competing human preferences.

\section*{Limitations}

Our work focuses on the multi-party travel planning domain as a representative case of multi-stakeholder alignment. While we believe the theoretical insights generalize to other multi-stakeholder settings (e.g., negotiation, collaborative writing, group recommendation), empirical validation in these domains remains for future work. Additionally, our programmatic verification layers rely on the availability of structured tool APIs; in domains where such tools are unavailable, a larger fraction of reward signal would necessarily depend on LLM-based estimation. Finally, existing benchmarks for multi-stakeholder planning are scarce, which limits the scale of end-to-end GRPO training experiments we can conduct; constructing richer multi-party benchmarks and running full-scale RL validation is an important direction for future work.

\bibliography{custom}

\appendix

\section{Multi-Stakeholder Reward Consistency Experiment Details}
\label{app:setup}

\paragraph{Scope.} This appendix section supports the reward-consistency analysis in \S\ref{sec:empirical}. We study multi-stakeholder travel planning, where a single itinerary must satisfy multiple travelers with heterogeneous constraints such as budget limits, dietary restrictions, mobility needs, safety constraints, and activity preferences. The response content is fixed, and only semantics-preserving presentation variants are scored.

\paragraph{Base responses.} We use five real multi-stakeholder travel queries, each expanded to $n \in \{2,3,5,8\}$ stakeholders. For each query--$n$ pair, we construct one high-, medium-, and low-quality response, yielding $5 \times 4 \times 3 = 60$ seed responses.

\paragraph{Quality bands.} High-quality responses satisfy hard constraints, maintain credible travel information, preserve time feasibility, and cover stakeholder needs in a balanced way. Medium-quality responses are broadly feasible but contain one or two compromises, questionable details, or under-served stakeholders. Low-quality responses contain hard-constraint violations, hallucinated or implausible travel details, or time-line conflicts. These quality bands are pre-constructed and used only as grouping variables; the consistency experiment does not generate additional quality-changing variants.

\subsection{Response Construction and Semantics-Preserving Variants}
\label{app:variants}

\paragraph{Generation.} For each base response, we first use DeepSeek-V4-Flash with tool calls to extract a structured JSON metadata file recording all factual content (locations, times, prices, activities), per-stakeholder constraint satisfaction status, and trade-off decisions. Rule-based variants (e.g., order permutations, evidence placement, formatting changes) operate directly on this metadata to guarantee semantic preservation. LLM-based rewrites are also performed by DeepSeek-V4-Flash, conditioned on the metadata file so that the rewrite cannot introduce information beyond---or omit information present in---the structured record.

\paragraph{Semantic-preservation constraints.} All variants preserve five invariants: (i) itinerary facts such as locations, times, prices, transportation, and activities; (ii) each hard or soft constraint's satisfaction status; (iii) the stakeholder set; (iv) the semantic strength of satisfaction or violation claims; and (v) the response quality band. Variants change only response presentation.

\paragraph{Variant families.} For each base response, we generate eleven semantics-preserving variant families with five versions per family. The exact-repeat baseline is collected separately and is not counted as a variant family.

\begin{table*}[ht]
\centering
\small
\begin{tabular}{@{}llp{9.5cm}@{}}
\toprule
\textbf{Variant family} & \textbf{Method} & \textbf{Operation} \\
\midrule
\multicolumn{3}{@{}l}{\textit{Multi-stakeholder-specific}} \\[2pt]
\quad Stakeholder section order & Rule & Permute per-stakeholder coverage sections in the response. \\
\quad Satisfaction summary order & Rule & Permute stakeholder entries in the satisfaction summary. \\
\quad Causal direction & LLM & Flip trade-off framing between cause$\to$effect and effect$\to$cause. \\
\quad Under-served evidence placement & LLM & Relocate under-served stakeholder evidence to different positions (before/after itinerary, interleaved, etc.). \\
\quad Shared-benefit attribution & Rule & Reorder multi-stakeholder attribution when a shared item serves multiple users. \\
\midrule
\multicolumn{3}{@{}l}{\textit{General presentation controls}} \\[2pt]
\quad Trade-off order & Rule & Permute conflict-explanation paragraphs without changing claims. \\
\quad Paraphrase & LLM & Rewrite wording with synonymous expressions; preserve all facts. \\
\quad Format conversion & Rule & Convert among prose, bullet lists, numbered lists, and markdown tables. \\
\quad Section header & Rule & Replace section headings with synonymous labels (e.g., ``Budget'' $\to$ ``Cost''). \\
\quad Low-level formatting & Rule & Change whitespace, bullet style, emphasis markers only. \\
\quad Topic position & LLM & Move non-stakeholder topical sections to different locations in the response. \\
\bottomrule
\end{tabular}
\caption{Semantics-preserving variant families. \textbf{Rule}: deterministic transforms operating on extracted metadata (order permutation, formatting changes). \textbf{LLM}: DeepSeek-V4-Flash rewrites conditioned on structured metadata to prevent fact drift. Five variants per family per response; exact repeats are a separate baseline.}
\label{tab:variants}
\end{table*}

\paragraph{Exact-repeat baseline.} Each base response is also scored five times with the exact same input (no rewriting). These repeats estimate intrinsic judge noise ($\sigma^2_{\text{Rep}}$).

\paragraph{Scale.} Each base response yields 60 scoring inputs ($11$ variant families $\times$ $5$ versions $+$ $5$ exact repeats). With 60 seeds, this gives $60 \times 60 = 3{,}600$ scoring inputs per judge model and scoring method.

\paragraph{Example.} We illustrate with a real 3-stakeholder query (Xiamen$\to$Fuzhou weekend trip, budget $\leq$500 CNY/person). The three travelers are: user1 (photography, prefers quiet spots), user2 (limited mobility, no hiking), and user3 (tight schedule, wants popular check-in spots). A \emph{stakeholder section order} variant permutes only the discussion order:

\begin{quote}
\small
\textbf{Original:} user1 (photographer) $\to$ user2 (mobility) $\to$ user3 (check-in) \\
\textbf{Variant:} user3 (check-in) $\to$ user1 (photographer) $\to$ user2 (mobility)
\end{quote}

\noindent The itinerary, constraint satisfaction, and all factual content remain identical. Yet this reordering alone causes score shifts of up to 2.0 points on the deterministic Gemini judge (3.5$\to$5.5) and a 4.8-point spread on DeepSeek (3.5--8.3 across five permutations), demonstrating that \emph{who is discussed first} influences the judge's holistic assessment even when all other content is fixed.

\subsection{Judge Scoring Protocols}
\label{app:rubric}

We evaluate four scoring protocols of increasing structural decomposition. All protocols share a verification-oriented system prompt that enforces evidence-based assessment: the judge must (i)~distinguish between surface mention and concrete executable arrangements, (ii)~independently verify budget totals by summing itemized costs rather than accepting self-reported figures, and (iii)~cross-check transit times against distance plausibility. All outputs use structured JSON for deterministic parsing.

\paragraph{Direct (holistic).} The judge conducts a chain-of-thought analysis along five predefined aspects---hard-constraint satisfaction, information credibility, timeline feasibility, inter-stakeholder fairness, and actionability---then produces a single scalar score $R \in [1.0, 10.0]$. Explicit scoring anchors map qualitative assessments to numeric ranges (e.g., 9.0--10.0: all constraints satisfied with fully credible information; 3.0--4.9: hard-constraint violations or substantial fabrication). The judge must output per-aspect evidence prior to the final score, ensuring the holistic judgment is grounded in verifiable observations.

\paragraph{Rubric (dimension-decomposed).} The judge independently scores five predefined quality dimensions (Table~\ref{tab:rubric}), each on $[1.0, 10.0]$ with dimension-specific anchors. The final score is the deterministic weighted sum $R = \sum_{d=1}^5 \alpha_d \cdot s_d$, where $\alpha_d$ are fixed weights and $s_d$ are per-dimension scores. This protocol tests whether decomposing the evaluation into orthogonal quality axes reduces presentation-induced variance by constraining the judge's attention.

\begin{table}[t]
\centering
\small
\resizebox{\columnwidth}{!}{%
\begin{tabular}{@{}lcp{4.8cm}@{}}
\toprule
\textbf{Dimension} & \textbf{Weight $\alpha_d$} & \textbf{Evaluation focus} \\
\midrule
Constraint satisfaction & 0.30 & Per-user hard constraints verified item-by-item; budget independently summed \\
Information credibility & 0.25 & POI existence, price plausibility, transit time--distance consistency \\
Timeline feasibility & 0.20 & Schedule coherence, transit time accounting, activity density \\
Fairness & 0.15 & Dedicated time slots per user; compensation for compromised stakeholders \\
Conflict resolution & 0.10 & Specificity of resolution strategies; presence of creative win-win solutions \\
\bottomrule
\end{tabular}
}
\caption{Rubric dimensions and fixed aggregation weights. Each dimension is scored on $[1.0, 10.0]$ with explicit anchors (e.g., constraint satisfaction: 10 = all satisfied, 6 = one minor violation, 2 = three or more violations).}
\label{tab:rubric}
\end{table}

\paragraph{Checklist (binary verification).} The judge answers a set of binary verification questions organized into five weighted categories: (A)~per-user hard-constraint checks ($\omega = 2.5$ per item), (B)~information credibility ($\omega = 2.0$), (C)~timeline logic ($\omega = 1.5$), (D)~fairness ($\omega = 1.5$), and (E)~actionability ($\omega = 1.0$). Each item requires cited evidence from the plan. The final score is computed deterministically:
\begin{equation}
R = \frac{\sum_{b} a_b \cdot \omega_b}{\sum_{b} \omega_b} \times 9 + 1, \quad R \in [1, 10]
\end{equation}
where $a_b \in \{0, 1\}$ is the binary answer and $\omega_b$ the item weight. Crucially, Category~A is instantiated per-stakeholder (3 items $\times$ $n$ users), so checklist length grows linearly with $n$.

\paragraph{Decomposed (adaptive).} The judge independently evaluates each stakeholder $i$ through a structured pipeline: (a)~extracting hard constraints and soft preferences from the user profile, (b)~verifying each constraint against the plan with cited evidence, (c)~assessing credibility of POIs relevant to that user, and (d)~assigning a satisfaction score $\hat{u}_i \in [1.0, 10.0]$ with explicit anchors. The judge additionally generates adaptive importance weights $w_i$ based on constraint severity (safety-related constraints receive highest weight; vague preferences receive lowest), subject to $\sum_i w_i = 1$. The final score is $R = \sum_i w_i \hat{u}_i$. This protocol serves dual purposes: as a structured scoring method and as a diagnostic instrument for measuring weight instability (\S\ref{sec:exp_implicit_weights}).

\paragraph{Decomposed (uniform).} To isolate the effect of weight drift from satisfaction estimation noise, we additionally report a post-hoc recomputation using uniform weights $w_i = 1/n$ applied to the \emph{same} $\hat{u}_i$ values produced by the adaptive protocol. This baseline requires no additional judge calls---it simply replaces the adaptive weights with equal weights in the aggregation formula $R = \frac{1}{n}\sum_i \hat{u}_i$. The comparison between adaptive and uniform variants (Table~\ref{tab:structured}) directly quantifies the variance contribution of weight assignment.

\paragraph{Judge models and decoding.} We use three judge families: Qwen3.5-Plus \citep{qwen2026qwen35}, Gemini-3-Flash \citep{google2025gemini3flash}, and DeepSeek-V4-Flash \citep{deepseekai2026deepseekv4}. All use greedy decoding ($T{=}0$) and structured JSON output. Gemini at $T{=}0$ is fully deterministic (identical inputs produce identical scores), providing a controlled setting in which all observed score variance is attributable solely to input presentation differences rather than sampling stochasticity.

\subsection{Reward Consistency Metrics}
\label{app:rewrite_analysis}

For an experimental unit $s=(\text{base response}, n, \text{quality})$, variant family $t$, variant version $b$, and exact repeat $r$, let $\mathrm{score}(s,t,b)$ be the parsed scalar judge score.

\paragraph{Semantic-variant variance.}
\begin{equation}
    \sigma^2_{\text{Sem}}(s)
    = \mathrm{Var}_{t,b}\!\left[\mathrm{score}(s,t,b)\right].
\end{equation}
This measures presentation-induced reward inconsistency: $\sigma^2_{\text{Sem}}(s)=0$ means the judge gives identical scores to all semantics-preserving variants of the same response, while larger values mean the score is sensitive to wording, ordering, or evidence placement.

\paragraph{Exact-repeat variance.}
\begin{equation}
    \sigma^2_{\text{Rep}}(s)
    = \mathrm{Var}_{r}\!\left[\mathrm{score}(s,\mathrm{repeat},r)\right].
\end{equation}

\paragraph{Coefficient of variation and scaling.}
\begin{align*}
    \mathrm{CV}(s)
    &= \frac{\sqrt{\sigma^2_{\text{Sem}}(s)}}
            {\mathrm{Mean}_{\text{Sem}}(s)}, \\
    \mathrm{Growth}
    &= \frac{\sigma^2_{\text{Sem}}(n=8)}
            {\sigma^2_{\text{Sem}}(n=2)} .
\end{align*}
We also report $\sigma^2_{\text{Sem}}/\sigma^2_{\text{Rep}}$; values above one indicate that semantics-preserving presentation changes induce more score variance than exact repeated scoring.

\paragraph{Decomposed-diagnostic metrics.} For Decomposed (adaptive) outputs, we measure $\mathrm{CV}(w_i)$, $\mathrm{CV}(\hat{u}_i)$, and the weight-induced score shift $\Delta R_w = \sum_i (w_i - \bar{w}_i)\hat{u}_i$, where $\bar{w}_i$ is the mean adaptive weight within the same query--quality--$n$ semantic-variant group.

\begin{table}[ht]
\centering
\resizebox{\columnwidth}{!}{%
\begin{tabular}{@{}l cccc r@{}}
\toprule
Stakeholder & $w_{\text{actual}}$ & $\bar{w}$ & $\Delta w$ & $\hat{u}_i$ & $\Delta w \times \hat{u}_i$ \\
\midrule
User 1 & 0.100 & 0.126 & $-$0.027 & 8.0 & $-$0.212 \\
User 2 & 0.100 & 0.139 & $-$0.040 & 8.0 & $-$0.316 \\
User 3 & 0.300 & 0.228 & $+$0.072 & 1.0 & $+$0.072 \\
User 4 & 0.300 & 0.251 & $+$0.049 & 2.0 & $+$0.098 \\
User 5 & 0.200 & 0.255 & $-$0.055 & 6.0 & $-$0.330 \\
\midrule
\multicolumn{5}{@{}l}{Total shift} & $\mathbf{-0.688}$ \\
\bottomrule
\end{tabular}}
\caption{Weight-induced score shift in a high-dispersion variant ($n{=}5$, high quality). Adaptive weights move mass away from high-$\hat{u}_i$ stakeholders and toward low-$\hat{u}_i$ stakeholders, producing a $-0.69$ aggregation shift.}
\label{tab:multiplier_example}
\end{table}

\subsection{Variance Breakdown by Variant Type}
\label{app:variant_breakdown}

Table~\ref{tab:variant_breakdown} decomposes $\sigma^2_{\text{Sem}}$ by variant family (all three judges, direct scoring, $T{=}0$). Two patterns stand out. First, variants that alter stakeholder salience or trade-off framing tend to produce the largest high-$n$ variance: \emph{causal\_direction} (0.93 at $n{=}8$), \emph{underserved\_evidence\_placement} (0.90), \emph{stakeholder\_coverage\_order} (0.85), and \emph{tradeoff\_explanation\_order} (Growth${=}1.58$). Second, low-level controls such as \emph{low\_level\_formatting} (Growth${=}1.02$) and \emph{section\_header\_variant} (Growth${=}1.07$) show much weaker scaling with $n$. This suggests that the $n$-dependent variance growth documented in \S\ref{sec:exp_scaling} is driven less by generic formatting and more by presentation changes that affect the judge's implicit weighting of stakeholder interests and trade-offs.

\begin{table*}[ht]
\centering
\small
\begin{tabular}{@{}l cccc c@{}}
\toprule
\textbf{Variant family} & $n{=}2$ & $n{=}3$ & $n{=}5$ & $n{=}8$ & \textbf{Growth} \\
\midrule
\multicolumn{6}{@{}l}{\textit{Multi-stakeholder-specific}} \\
\quad stakeholder\_coverage\_order & --- & 0.76 & 0.71 & 0.85 & --- \\
\quad satisfaction\_summary\_order & --- & 0.73 & 0.79 & 0.55 & --- \\
\quad causal\_direction            & 0.75 & 0.89 & 1.18 & 0.93 & 1.24 \\
\quad shared\_benefit\_attribution & 0.49 & 0.56 & 0.77 & 0.54 & 1.11 \\
\quad underserved\_evidence\_placement & 0.76 & 0.60 & 1.34 & 0.90 & 1.18 \\
\midrule
\multicolumn{6}{@{}l}{\textit{General presentation controls}} \\
\quad tradeoff\_explanation\_order  & 0.39 & 0.54 & 0.95 & 0.62 & 1.58 \\
\quad generic\_paraphrase          & 0.68 & 0.69 & 0.82 & 0.84 & 1.23 \\
\quad format\_conversion           & 0.56 & 0.50 & 0.80 & 0.68 & 1.22 \\
\quad section\_header\_variant     & 0.55 & 0.69 & 0.77 & 0.58 & 1.07 \\
\quad low\_level\_formatting       & 0.51 & 0.54 & 0.87 & 0.52 & 1.02 \\
\quad topic\_position              & 0.75 & 1.19 & 0.86 & 0.64 & 0.86 \\
\midrule
exact\_repeat (baseline)           & 0.26 & 0.54 & 0.49 & 0.34 & 1.33 \\
\midrule
\textbf{ALL (semantic)}            & 0.64 & 0.68 & 0.96 & 0.86 & 1.34 \\
\bottomrule
\end{tabular}
\caption{$\sigma^2_{\text{Sem}}$ by variant family (all three judges, direct scoring, $T{=}0$). Growth = $\sigma^2(n{=}8)/\sigma^2(n{=}2)$. ``---'' indicates the variant is undefined for $n{=}2$ (requires ${\geq}3$ stakeholders). The exact\_repeat row estimates intrinsic judge noise $\sigma^2_{\text{Rep}}$; ALL aggregates the eleven semantic variant families.}
\label{tab:variant_breakdown}
\end{table*}

\section{Theory Details for Section~\ref{sec:theory}}
\label{app:theory_details}

\subsection{Proof of Theorem~\ref{thm:decomposition} (Decomposition)}
\label{app:proof_decomposition}

\begin{proof}
Substituting $\hat{w}_i = w_i^* + \eta_i$ and $\hat{u}_i = u_i^* + \delta_i$ into the holistic judge model $\hat{R}(y) = \sum_i \hat{w}_i \hat{u}_i + \epsilon$:
\begin{align*}
    \hat{R}(y) &= \sum_i (w_i^* + \eta_i)(u_i^* + \delta_i) + \epsilon \nonumber \\
    &= \underbrace{\textstyle\sum_i w_i^* u_i^*}_{R^*(y)} + \underbrace{\textstyle\sum_i w_i^* \delta_i}_{\text{(I)}} \nonumber \\
    &\quad + \underbrace{\textstyle\sum_i \eta_i u_i^*}_{\text{(II)}} + \underbrace{\textstyle\sum_i \eta_i \delta_i}_{\text{cross}} + \epsilon
\end{align*}
Since $R^*(y)$ is a constant for fixed $y$, all noise terms have zero mean ($E[\delta_i] = E[\eta_i] = E[\epsilon] = 0$), and the estimation noise, weight noise, and residual noise are mutually independent, the pairwise covariance terms among (I), (II), the interaction term, and $\epsilon$ vanish:
\begin{align}
    \mathrm{Var}[\hat{R}(y)] &= \underbrace{\mathrm{Var}\!\left[\textstyle\sum_i w_i^* \delta_i\right]}_{\text{Term I}} + \underbrace{\mathrm{Var}\!\left[\textstyle\sum_i \eta_i u_i^*\right]}_{\text{Term II}} \nonumber \\
    &\quad + \underbrace{\mathrm{Var}\!\left[\textstyle\sum_i \eta_i \delta_i\right]}_{C_{\text{cross}}} + \underbrace{\sigma_\epsilon^2}_{\text{Term III}}
\end{align}
Term~I $= \sum_i (w_i^*)^2 \sigma_{\delta_i}^2$ by mutual independence of $\delta_i$. Term~II is left in general form, as $\eta_i$ may be correlated through the normalization constraint $\sum_i \hat{w}_i = 1$ (expanded in Proposition~\ref{prop:scaling}). For $C_{\text{cross}}$, by the independence assumption $\boldsymbol{\delta} \perp \boldsymbol{\eta}$ (Theorem~\ref{thm:decomposition}), $\delta_i$ is independent of all $\eta_j$, and $E[\delta_i]{=}0$ eliminates cross-$i$ terms:
\begin{equation}
    \mathrm{Var}\!\left[\textstyle\sum_i \eta_i \delta_i\right]
    = \sum_i E[\eta_i^2]\,E[\delta_i^2]
    = \sum_i \sigma_{\eta_i}^2 \sigma_{\delta_i}^2.
\end{equation}
This yields Eq.~\ref{eq:decomposition}.

An equivalent derivation uses the Law of Total Variance, conditioning on the implicit weights $\boldsymbol{\eta} = (\eta_1, \ldots, \eta_n)$:
\begin{align}
    \mathrm{Var}(\hat{R}) &= \underbrace{E_{\boldsymbol{\eta}}[\mathrm{Var}(\hat{R} \mid \boldsymbol{\eta})]}_{\text{Terms I + III + cross}} \nonumber \\
    &\quad + \underbrace{\mathrm{Var}_{\boldsymbol{\eta}}(E[\hat{R} \mid \boldsymbol{\eta}])}_{\text{Term II}}
\end{align}
Here,
\begin{align}
    &E_{\boldsymbol{\eta}}[\mathrm{Var}(\hat{R} \mid \boldsymbol{\eta})] \nonumber \\
    &\quad = E_{\boldsymbol{\eta}}\!\left[\textstyle\sum_i (w_i^*+\eta_i)^2\sigma_{\delta_i}^2 + \sigma_\epsilon^2\right] \nonumber \\
    &\quad = \textstyle\sum_i (w_i^*)^2\sigma_{\delta_i}^2 + \sum_i \sigma_{\eta_i}^2\sigma_{\delta_i}^2 + \sigma_\epsilon^2 ,
\end{align}
recovering Term~I, the interaction term, and Term~III. The inner expectation $E[\hat{R} \mid \boldsymbol{\eta}] = \sum_i (w_i^* + \eta_i) u_i^*$ varies only through $\eta_i$, recovering Term~II as the variance of the aggregation step.
\end{proof}

\subsection{Proof of Proposition~\ref{prop:scaling} (Scaling)}
\label{app:proof_scaling}

\paragraph{Full Derivation of Proposition~\ref{prop:scaling}.}

The weight drift $\eta_i$ satisfies $\sum_i \eta_i = 0$ (normalization constraint: $\sum_i \hat{w}_i = \sum_i w_i^* = 1$). Proposition~\ref{prop:scaling} assumes exchangeable zero-sum drift for $n\ge2$: $\mathrm{Var}(\eta_i)=\sigma_\eta^2(n)$ for all $i$ and a common off-diagonal covariance $\mathrm{Cov}(\eta_i,\eta_j)=c_\eta$ for $i\neq j$. Since $\sum_i \eta_i = 0$ holds identically (not merely in expectation), the variance of this constant is zero. Expanding under exchangeability:
\begin{align}
    0 &= \mathrm{Var}\!\left[\textstyle\sum_i \eta_i\right]
      = \sum_i \mathrm{Var}[\eta_i] + \sum_{i\neq j}\mathrm{Cov}(\eta_i,\eta_j) \nonumber \\
      &= n\sigma_\eta^2(n) + n(n-1)c_\eta,
\end{align}
so $c_\eta=-\sigma_\eta^2(n)/(n-1)$. Expanding Term~II using the standard variance formula for a linear combination:
\begin{align*}
    \text{Term II} &= \mathrm{Var}\!\left[\textstyle\sum_i \eta_i u_i^*\right] \nonumber \\
    &= \sum_i \sum_j u_i^* u_j^* \,\mathrm{Cov}(\eta_i, \eta_j) \nonumber \\
    &= \sum_i (u_i^*)^2 \mathrm{Var}(\eta_i) \nonumber \\
    &\quad + \sum_{i\neq j} u_i^* u_j^*
       \,\mathrm{Cov}(\eta_i, \eta_j) \nonumber \\
    &= \sigma_\eta^2(n)\sum_i (u_i^*)^2 \nonumber \\
    &\quad - \frac{\sigma_\eta^2(n)}{n-1}
       \sum_{i\neq j}u_i^*u_j^*
\end{align*}
Using the identity $\sum_{i\neq j}u_i^*u_j^* = (\sum_i u_i^*)^2 - \sum_i(u_i^*)^2$ to eliminate the off-diagonal sum:
\begin{align*}
    &= \sigma_\eta^2(n)\sum_i (u_i^*)^2 \nonumber \\
    &\quad - \frac{\sigma_\eta^2(n)}{n-1}
       \!\left[(\textstyle\sum_i u_i^*)^2
       - \sum_i(u_i^*)^2\right] \nonumber \\
    &= \sigma_\eta^2(n)\!\left(1+\tfrac{1}{n-1}\right)
       \sum_i (u_i^*)^2 \nonumber \\
    &\quad - \frac{\sigma_\eta^2(n)}{n-1}
       (\textstyle\sum_i u_i^*)^2 \nonumber \\
    &= \frac{n\,\sigma_\eta^2(n)}{n-1} \nonumber \\
    &\quad \cdot
       \left[\sum_i (u_i^*)^2
       -\frac{(\sum_i u_i^*)^2}{n}\right]
\end{align*}
Recognizing the bracketed expression as the (unnormalized) variance $\sum_i(u_i^*-\bar{u})^2 = \sum_i(u_i^*)^2 - \frac{1}{n}(\sum_i u_i^*)^2$, and applying $\mathrm{Var}_i[u_i^*]\triangleq \frac{1}{n}\sum_i(u_i^*-\bar{u})^2$:
\begin{align}
    &= \frac{n\,\sigma_\eta^2(n)}{n-1}
       \sum_i(u_i^*-\bar{u})^2 \nonumber \\
    &= \frac{n^2}{n-1}\,\sigma_\eta^2(n)\,\mathrm{Var}_i[u_i^*].
\end{align}

Here $\bar{u}=\frac{1}{n}\sum_i u_i^*$. The formula exposes three factors: weight-drift variance $\sigma_\eta^2(n)$, utility dispersion $\mathrm{Var}_i[u_i^*]$, and the count factor $n^2/(n-1)$. Term~II vanishes if either weight drift or utility dispersion is zero; if either is small, the term is proportionally small. For fixed utility dispersion, the $n$-dependence is governed by $n\sigma_\eta^2(n)$, so the term grows only when per-stakeholder drift does not shrink fast enough and remains bounded when $\sigma_\eta^2(n)=O(1/n)$.

\paragraph{Remark on exchangeability.}
The exchangeable model is a symmetric specialization. Without exchangeability, let $\Sigma_\eta=\mathrm{Cov}(\boldsymbol{\eta})$; the zero-sum constraint implies $\Sigma_\eta\mathbf{1}=0$. Then
\begin{equation}
    \text{Term~II}
    = (u^*)^\top \Sigma_\eta u^*
    = (u^*-\bar{u}\mathbf{1})^\top
      \Sigma_\eta
      (u^*-\bar{u}\mathbf{1}).
\end{equation}
Thus the qualitative condition is unchanged: aggregation noise vanishes when weight drift is zero or when stakeholder utilities are constant, and otherwise depends on the alignment between utility dispersion and the covariance directions of weight drift. Proposition~\ref{prop:scaling} is the exchangeable special case, where $\Sigma_\eta=\frac{n}{n-1}\sigma_\eta^2(n)(I-\mathbf{1}\mathbf{1}^\top/n)$.

\paragraph{Relation to Term I.}

For any fixed target weights $w_i^*$ with $\sum_i w_i^* = 1$:
\begin{align}
    \text{Term I} &= \sum_{i=1}^n (w_i^*)^2 \sigma_{\delta_i}^2 .
\end{align}
Under uniform weights ($w_i^* = 1/n$) and homogeneous estimation noise ($\sigma_{\delta_i}^2 = \sigma_\delta^2(n)$), this simplifies to $\sigma_\delta^2(n)/n$. Thus Term~I reflects averaged utility-estimation noise, while Term~II reflects aggregation noise multiplied by satisfaction dispersion. This is why Term~II can create large local score shifts in high-dispersion groups even when satisfaction estimation explains most average variance.

\subsection{Irreducibility Analysis}
\label{app:irreducibility}

\paragraph{Why Term~II persists.} A holistic judge that outputs a scalar $\hat{R}$ for a multi-stakeholder response can be modeled as implicitly computing weights $\hat{w}_i$ (the marginal influence of each stakeholder's satisfaction on the final score). These weights encode a normative judgment---whose needs matter more---for which there is no objective target to infer from the response alone. In stochastic LLM judging, when no explicit constraint fixes $\hat{w}_i = w_i^*$, repeated evaluations can yield varying implicit weights. Under the symmetric noise model of Proposition~\ref{prop:scaling}, $\mathrm{Cov}(\boldsymbol{\eta})$ has rank $n{-}1$ (the full zero-sum subspace), so any non-constant satisfaction vector ($\mathrm{Var}_i[u_i^*] > 0$) has a nonzero projection onto the weight-noise subspace and yields a positive Term~II.

\paragraph{Rubric and Checklist.} These methods structure the \emph{estimation} task by decomposing it into dimensions (accuracy, feasibility, constraint satisfaction, etc.). However, these dimensions cut \emph{across} stakeholders rather than \emph{between} them. When scoring ``constraint satisfaction'' for a 4-person group, the judge still implicitly decides whose constraints matter more, pushing aggregation one level deeper. Formally, within each dimension $d$, the judge computes $\hat{R}_d = \sum_i \hat{w}_{i|d} \cdot \hat{u}_{i,d}$ with stochastic $\hat{w}_{i|d}$, preserving Term~II at the sub-dimension level.

\paragraph{Ensemble/Repeated Sampling.} Averaging $M$ independent judge calls reduces \emph{stochastic} variance by $1/M$: $\mathrm{Var}[\bar{R}_M] = \mathrm{Var}[\hat{R}]/M$. As $M \to \infty$, the stochastic component of Term~II vanishes, but $\bar{R}_M$ converges to the judge's expected aggregation rule, $E[\hat{R}] = \sum_i E[\hat{w}_i]u_i^*$, which need not equal the designer's intended $\sum_i w_i^*u_i^*$. The resulting bias cannot be removed by averaging.

\paragraph{Debate/Deliberation.} Multi-agent debate can surface reasoning that a single judge might miss, potentially reducing $\sigma_\delta^2$. However, unless the protocol explicitly separates per-role estimation from cross-role aggregation, the final arbiter still produces a holistic score through implicit stakeholder weights, so Term~II persists.

\subsection{Proof of Proposition~\ref{prop:snr_sign} (Advantage Sign Correctness)}
\label{app:snr_derivation}

\subsubsection{Part 1: Derivation of Eq.~\ref{eq:p_correct}}

\paragraph{Setup.} Consider a GRPO group of $G$ candidate responses with
target rewards $R_j^*$ and group mean
$\bar{R}^*=\frac{1}{G}\sum_{k=1}^{G}R_k^*$. The \textbf{true reward gap} for
candidate $j$ is
\begin{equation}
    \Delta_j^* = R_j^*-\bar{R}^* .
\end{equation}
The judge observes a noisy reward
\begin{equation}
    \hat{R}_j = R_j^*+\xi_j ,
\end{equation}
where the reward errors are independent Gaussian variables
$\xi_j\sim\mathcal{N}(0,\sigma_\xi^2)$. GRPO uses the group-relative observed
gap
\begin{align}
    \hat{\Delta}_j
    &= \hat{R}_j-\bar{\hat{R}} \nonumber \\
    &= (R_j^*-\bar{R}^*) + (\xi_j-\bar{\xi}) \nonumber \\
    &= \Delta_j^* + \zeta_j ,
\end{align}
where $\zeta_j=\xi_j-\bar{\xi}$ is the group-centered \textbf{random error}.
Thus $\hat{\Delta}_j$ is a noisy reward gap: the desired signal
$\Delta_j^*$ plus judge instability.

\paragraph{Noise after group centering.} Since
$\bar{\xi}=\frac{1}{G}\sum_{k=1}^{G}\xi_k$, the group-centered error expands
as
\begin{align}
    \zeta_j
    &= \xi_j-\bar{\xi} \nonumber \\
    &= \xi_j-\frac{1}{G}\sum_{k=1}^{G}\xi_k \nonumber \\
    &= \xi_j-\frac{1}{G}
       \left(\xi_j+\sum_{k\ne j}\xi_k\right) \nonumber \\
    &= \left(1-\frac{1}{G}\right)\xi_j
       -\frac{1}{G}\sum_{k\ne j}\xi_k .
\end{align}
For independent errors, all covariance terms vanish. Hence
\begin{align}
    \mathrm{Var}(\zeta_j)
    &= \mathrm{Var}\!\left(
       \left(1-\frac{1}{G}\right)\xi_j
       -\frac{1}{G}\sum_{k\ne j}\xi_k
       \right) \nonumber \\
    &= \left(1-\frac{1}{G}\right)^2\mathrm{Var}(\xi_j)
       +\frac{1}{G^2}\sum_{k\ne j}\mathrm{Var}(\xi_k) \nonumber \\
    &= (1-\tfrac{1}{G})^2\sigma_\xi^2
       +(G-1)\tfrac{1}{G^2}\sigma_\xi^2 \nonumber \\
    &= \sigma_\xi^2(1-1/G).
\end{align}
Because $\zeta_j$ is a linear combination of independent Gaussian errors,
\begin{equation}
    \zeta_j \sim
    \mathcal{N}\!\left(0,\sigma_\xi^2(1-1/G)\right).
\end{equation}

\paragraph{Sign correctness.} Assume first that $\Delta_j^*>0$; the case
$\Delta_j^*<0$ is symmetric. GRPO assigns the correct advantage sign when the
observed gap remains positive:
\begin{align}
    P(\mathrm{correct\ sign})
    &= P(\hat{\Delta}_j>0) \nonumber \\
    &= P(\Delta_j^*+\zeta_j>0) \nonumber \\
    &= P(\zeta_j>-\Delta_j^*) .
\end{align}
Standardize the centered noise as
\begin{equation}
    Z_j =
    \frac{\zeta_j}{\sigma_\xi\sqrt{1-1/G}}
    \sim \mathcal{N}(0,1).
\end{equation}
Then
\begin{align}
    P(\mathrm{correct\ sign})
    &= P\!\left(
       Z_j >
       -\frac{\Delta_j^*}{\sigma_\xi\sqrt{1-1/G}}
       \right) \nonumber \\
    &= \Phi\!\left(
       \frac{\Delta_j^*}{\sigma_\xi\sqrt{1-1/G}}
       \right).
\end{align}
For $\Delta_j^*<0$, the same expression holds with $|\Delta_j^*|$.

\paragraph{Evaluating at a typical response.} To obtain a single summary
statistic, evaluate a response whose true reward gap has magnitude one group
standard deviation:
\begin{equation}
    |\Delta_j^*|=\sqrt{\mathrm{Var}_j[\Delta_j^*]}.
\end{equation}
Using the SNR definition in Eq.~\ref{eq:snr_def},
$\mathrm{SNR}=\mathrm{Var}_j[\Delta_j^*]/\sigma_\xi^2$, gives
\begin{equation}
    P(\text{correct sign})
    =
    \Phi\!\left(\frac{\sqrt{\mathrm{SNR}}}{\sqrt{1-1/G}}\right).
    \label{eq:p_correct_full}
\end{equation}
This represents an optimistic estimate: the majority of responses in a group
lie closer to the mean ($|\Delta_j^*| < \mathrm{sd}$) and experience
correspondingly higher sign-flip rates.

\subsubsection{Part 2: Conditional SNR Scaling with Stakeholder Count}

This part analyzes the conditional limit where Term~II is the leading contribution to reward-error variance. This is not assumed to be the average empirical regime; it isolates how aggregation noise affects GRPO when satisfaction dispersion is high. In this limit, $\sigma_\xi^2 \approx \mathrm{Term~II}$. By Proposition~\ref{prop:scaling}:
\begin{equation}
    \sigma_\xi^2(n) \approx \frac{n^2}{n-1}\,\sigma_\eta^2(n)\,\mathrm{Var}_i[u_i^*]
\end{equation}
Substituting into the SNR definition:
\begin{align}
    \mathrm{SNR}(n)
    &= \frac{\mathrm{Var}_j[\Delta_j^*]}{\sigma_\xi^2(n)} \nonumber \\
    &\approx \frac{(n{-}1)\,\mathrm{Var}_j[\Delta_j^*]}{n^2\,\sigma_\eta^2(n)\,\mathrm{Var}_i[u_i^*]}
\end{align}
Thus SNR is controlled by the same factors in the denominator. In the high-drift, high-dispersion regime where $\sigma_\eta^2(n)=\Theta(1)$, $\mathrm{Var}_i[u_i^*]=\Theta(1)$, and the true reward-gap spread $\mathrm{Var}_j[\Delta_j^*]$ is bounded, this is $O(1/n)$; if either factor shrinks with $n$, degradation is weaker, and if either vanishes the Term~II channel disappears. Eq.~\ref{eq:p_correct_full} depends on $\sqrt{\mathrm{SNR}}$, so its argument is $O(1/\!\sqrt{n})$ in the high-drift, high-dispersion regime. Using $\Phi(x)=1/2+O(x)$ gives $P(\text{correct sign}) = 1/2 + O(1/\!\sqrt{n})$.

\subsubsection{\texorpdfstring{Numerical Thresholds ($G=8$)}{Numerical Thresholds (G=8)}}

With $1-1/G = 7/8$:
\begin{center}
\small
\begin{tabular}{@{}cc@{}}
\toprule
\textbf{SNR} & \textbf{$P(\text{correct sign})$} \\
\midrule
0.5 & 0.78 \\
1.0 & 0.86 \\
1.5 & 0.90 \\
2.0 & 0.93 \\
3.0 & 0.97 \\
4.0 & 0.98 \\
\bottomrule
\end{tabular}
\end{center}

Under ideal i.i.d.\ noise, even moderate SNR appears sufficient. However, this represents a best-case bound: in practice, (1)~noise is correlated across responses sharing the same prompt (reducing effective SNR), (2)~the judge's noise is not Gaussian but exhibits systematic biases (e.g., position bias, length bias), and (3)~the ``typical response'' assumption ($|\Delta_j| = \sigma_{\mathrm{signal}}$) is optimistic---many responses have smaller true differences. These factors explain why empirically observed training instability occurs at SNR values that the i.i.d.\ theory would predict as sufficient.

\section{End-to-End Training Details}
\label{app:training_details}

\paragraph{Base model and SFT.} We use Qwen3-30B-A3B-Thinking-2507 \citep{qwen2025qwen3} as the base model. Expert trajectories are generated by Qwen3.5-Plus on multi-stakeholder travel-planning queries and used for supervised fine-tuning before RL.

\paragraph{GRPO training.} All four reward configurations (\textsc{DecompR}, Direct, Rubric, Checklist) share the same GRPO training setup: 6{,}000 training queries, with the best checkpoint selected within 400 training steps. The training queries are distinct from the 60 seed responses used in the reward-consistency analysis (\S\ref{sec:empirical}) and from the 650 MR-TravelBench evaluation tasks.

\paragraph{Evaluation.} MR-TravelBench contains 650 held-out multi-stakeholder travel-planning tasks spanning varying stakeholder counts and constraint difficulties. Each policy is evaluated with 3 independent trials per task.

\end{document}